\documentclass{article} 
\usepackage{fullpage,times}

\usepackage{amsmath}
\usepackage{amssymb}
\usepackage{amsthm}
\usepackage{natbib}
\usepackage{graphicx}
\usepackage{tikz}
\usepackage{pgfplots}
\usepackage{url}

\newtheorem{theorem}{Theorem}
\newtheorem{lemma}{Lemma}
\newtheorem{claim}{Claim}

\renewcommand{\Pr}{\mathbb{P}}
\newcommand{\Ind}[1]{\mathbb{I}_{\{#1\}}}

\newcommand{\naturals}{\mathbb{N}}

\newcommand{\E}{\mathbb{E}}
\newcommand{\Sm}{\mathbb{S}}
\newcommand{\Qm}{\mathbb{Q}}

\newcommand{\argmin}[1]{\underset{#1}{\mathrm{argmin}}}
\newcommand{\argmax}[1]{\underset{#1}{\mathrm{argmax}}}
\newcommand{\KL}{\mathrm{D}_{\mathrm{\tiny KL}}}

\newcommand{\Acal}{\mathcal{A}}

\newcommand{\Ocal}{\mathcal{O}}

\newcommand{\fhat}{\widehat{f}}

\renewcommand{\eqref}[1]{Eq.~(\ref{#1})}
\newcommand{\lemref}[1]{Lemma~\ref{#1}}
\newcommand{\thmref}[1]{Thm.~\ref{#1}}
\newcommand{\appref}[1]{Appendix~\ref{#1}}

\title{Online Learning with Switching Costs and Other Adaptive Adversaries}

\author{\\Nicol\`{o} Cesa-Bianchi\\
DSI, Universit\`{a} degli Studi di Milano, Italy \\
\texttt{nicolo.cesa-bianchi@unimi.it}
\and
\\Ofer Dekel\\
Microsoft Research, Redmond WA, USA\\
\texttt{oferd@microsoft.com}
\and
\\Ohad Shamir\\
Microsoft Research and the Weizmann Institute\\
\texttt{ohad.shamir@weizmann.ac.il}
}

\date{}

\begin{document}

\maketitle

\begin{abstract}
We study the power of different types of adaptive (nonoblivious)
adversaries in the setting of prediction with expert advice, under
both full-information and bandit feedback. We measure the player's
performance using a new notion of regret, also known as policy regret,
which better captures the adversary's adaptiveness to the player's
behavior. In a setting where losses are allowed to drift, we
characterize ---in a nearly complete manner--- the power of adaptive
adversaries with bounded memories and switching costs. In particular,
we show that with switching costs, the attainable rate with bandit
feedback is $\widetilde{\Theta}(T^{2/3})$. Interestingly, this rate is
significantly worse than the $\Theta(\sqrt{T})$ rate attainable with
switching costs in the full-information case. Via a novel reduction
from experts to bandits, we also show that a bounded memory adversary
can force $\widetilde{\Theta}(T^{2/3})$ regret even in the full
information case, proving that switching costs are easier to control
than bounded memory adversaries. Our lower bounds rely on a new
stochastic adversary strategy that generates loss processes with
strong dependencies.
\end{abstract}

\section{Introduction}
An important instance of the framework of prediction with expert
advice ---see, e.g., \cite{CesaBianchi:06}--- is defined as the
following repeated game, between a randomized player with a finite and fixed set
of available actions and an adversary. At the beginning of each round
of the game, the adversary assigns a loss to each action. Next, the
player defines a probability distribution over the actions, draws an
action from this distribution, and suffers the loss associated with
that action. The player's goal is to accumulate loss at the smallest
possible rate, as the game progresses. Two versions of this game are
typically considered: in the \emph{full-information feedback} version
of the game, at the end of each round, the player observes the adversary's
assignment of loss values to each action. In the \emph{bandit feedback}
version, the player only observes the loss associated with his chosen action,
but not the loss values of other actions.

We assume that the adversary is \emph{adaptive} (also called
\emph{nonoblivious} by \cite{CesaBianchi:06} or \emph{reactive} by
\cite{Munos:10}), which means that the adversary chooses the loss
values on round $t$ based on the player's actions on rounds $1\ldots
t-1$. We also assume that the adversary is deterministic and has
unlimited computational power. These assumptions imply that the
adversary can specify his entire strategy before the game begins. In
other words, the adversary can perform all of the calculations needed
to specify, in advance, how he plans to react on each round to any
sequence of actions chosen by the player.

More formally, let $\Acal$ denote the finite set of actions and let $X_t$
denote the player's random action on round $t$.  We adopt the notation
$X_{1:t}$ as shorthand for the sequence $X_1 \ldots X_t$. We assume
that the adversary defines, in advance, a sequence of
\emph{history-dependent loss functions} $f_1, f_2,\ldots$. The input to
each loss function $f_t$ is the entire history of the player's actions
so far, therefore the player's loss on round $t$ is $f_t(X_{1:t})$.
Note that the player doesn't observe the functions $f_t$, only the
losses that result from his past actions. Specifically, in the bandit
feedback model, the player observes $f_t(X_{1:t})$ on round $t$,
whereas in the full-information model, the player observes
$f_t(X_{1:t-1},x)$ for all $x\in \Acal$.

On any round $T$, we evaluate the player's performance so far using
the notion of \emph{regret}, which compares his cumulative loss on the
first $T$ rounds to the cumulative loss of the best fixed action in
hindsight.  Formally, the player's regret on round $T$ is defined as
\begin{equation} \label{eqn:policyRegret}
R_T ~=~ \sum_{t=1}^T f_t(X_{1:t}) - \min_{x \in \Acal} \sum_{t=1}^T f_t(x\ldots x) ~~.
\end{equation}
$R_T$ is a random variable, as it depends on the randomized action
sequence $X_{1:t}$. Therefore, we also consider the \emph{expected
  regret} $\E[R_T]$. This definition is the same as the one used in
\cite{Merhav:02} and \cite{Arora:12} (in the latter, it is called
\emph{policy regret}), but differs from the more common definition of
expected regret
\begin{equation} \label{eqn:oldRegret}
\E\left[\sum_{t=1}^T f_t(X_{1:t}) - \min_{x \in \Acal} \sum_{t=1}^T f_t(X_{1:t-1},x)\right]~.
\end{equation}
The definition in \eqref{eqn:oldRegret} is more common in the
literature (e.g., \cite{Auer:02,McMahan:04,Dani:06,Munos:10}), but is
clearly inadequate for measuring a player's performance against an
adaptive adversary. Indeed, if the adversary is adaptive, the quantity
$f_t(X_{1:t-1},x)$
is hardly interpretable ---see \cite{Arora:12}
for a more detailed discussion.

In general, we seek algorithms for which $\E[R_T]$ can be bounded by a
sublinear function of $T$, implying that the per-round expected
regret, $\E[R_T]/T$, tends to zero. Unfortunately, \cite{Arora:12}
shows that arbitrary adaptive adversaries can easily force the
regret to grow linearly. Thus, we need to focus on (reasonably) weaker
adversaries, which have constraints on the loss functions
they can generate.

The weakest adversary we discuss is the \emph{oblivious adversary},
which determines the loss on round $t$ based only on the current
action $X_t$. In other words, this adversary is oblivious to the
player's past actions.  Formally, the oblivious adversary is
constrained to choose a sequence of loss functions that satisfies
$\forall t$, $\forall x_{1:t} \in \Acal^t$, and $\forall x'_{1:t-1} \in \Acal^{t-1}$,
\begin{equation}
f_t(x_{1:t}) ~=~ f_t(x'_{1:t-1}, x_t)~~.
\label{eqn:oblivious}
\end{equation}
The majority of previous work in online learning focuses on oblivious
adversaries. When dealing with oblivious adversaries, we denote the
loss function by $\ell_t$ and omit the first $t-1$ arguments. With
this notation, the loss at time $t$ is simply written as
$\ell_t(X_t)$.

For example, imagine an investor that invests in a single
stock at a time. On each trading day he invests in one stock and suffers losses
accordingly. In this example, the investor is the player and the stock
market is the adversary. If the investment amount is small, the
investor's actions will have no measurable effect on the market, so
the market is oblivious to the investor's actions. Also note that this
example relates to the full-information feedback version of the game, as the
investor can see the performance of each stock at the end of each
trading day.

A stronger adversary is the \emph{oblivious adversary with switching
  costs}.  This adversary is similar to the oblivious adversary defined
above, but charges the player an additional switching cost of~$1$
whenever $X_{t} \neq X_{t-1}$. More formally, this adversary defines
his sequence of loss functions in two steps: first he chooses an
oblivious sequence of loss functions, $\ell_1,\ell_2\ldots$, which
satisfies the constraint in \eqref{eqn:oblivious}.
Then, he sets
$f_1(x) = \ell_1(x)$, and
\begin{equation}
\label{eqn:switchingDef}
\forall~t\geq 2,~~~f_t(x_{1:t}) ~=~ \ell_t(x_t) + \Ind{x_t \neq x_{t-1}}~~.
\end{equation}
This is a very natural setting. For example, let us consider again
the single-stock investor, but now assume that each trade has a fixed
commission cost. If the investor keeps his position in a stock for
multiple trading days, he is exempt from any additional fees, but when
he sells one stock and buys another, he incurs a fixed
commission. More generally, this setting (or simple generalizations of
it) allows us to capture any situation where choosing a different
action involves a costly change of state. In the paper, we will also
discuss a special case of this adversary, where the loss function
$\ell_t(x)$ for each action is sampled i.i.d.\ from a fixed
distribution.

The switching costs adversary defines the loss on round $t$ as a
function of $X_t$ and $X_{t-1}$, and is therefore a special case of a
more general adversary called an \emph{adaptive adversary with a
  memory of $1$}. This adversary is constrained to choose loss
functions that satisfy $\forall t$, $\forall x_{1:t} \in \Acal^t$, and $\forall x'_{1:t-2} \in \Acal^{t-2}$,
\begin{equation}
f_t(x_{1:t}) ~=~ f_t(x'_{1:t-2}, x_{t-1},x_t)~~.
\label{eqn:onemem}
\end{equation}
This adversary is more general than the switching costs adversary because his loss
functions can depend on the previous action in an arbitrary way.
We can further strengthen this adversary and define the \emph{bounded
  memory adaptive adversary}, which has a bounded memory of an
arbitrary size. In other words, this adversary is allowed to set his
loss function based on the player's $m$ most recent past actions,
where $m$ is a predefined parameter.  Formally, the bounded memory
adversary must choose loss functions that satisfy, $\forall t$, $\forall x_{1:t} \in \Acal^t$, and $\forall x'_{1:t-m-1} \in \Acal^{t-m-1}$,
$$
f_t(x_{1:t}) ~=~ f_t(x'_{1:t-m-1},x_{t-m:t})~~.
$$
In the information theory literature, this setting is called
\emph{individual sequence prediction against loss functions with
  memory} \cite{Merhav:02}.

In addition to the adversary types described above, the bounded memory
adaptive adversary has additional interesting special cases. One of
them is the \emph{delayed feedback oblivious adversary} of
\cite{Mesterharm:05}, which defines an oblivious loss sequence, but
reveals each loss value with a delay of $m$ rounds.
Since the loss at time $t$ depends on the player's action at time
$t-m$, this adversary is a special case of a bounded memory adversary
with a memory of size $m$. The delayed feedback adversary is not a
focus of our work, and we present it merely as an interesting special
case.

So far, we have defined a succession of adversaries of different
strengths. This paper's goal is to understand the upper and lower
bounds on the player's regret when he faces these
adversaries. Specifically, we focus on how the expected regret depends
on the number of rounds, $T$, with either full-information or bandit
feedback.

\subsection{The Current State of the Art}

Different aspects of this problem have been previously studied and the
known results are surveyed below and summarized
in Table \ref{tbl:bounds}. Most of these previous results rely
on the additional assumption that the range of the loss functions is
bounded in a fixed interval, say $[0,C]$. We explicitly make note of this
because our new results require us to slightly generalize this
assumption.

As mentioned above, the oblivious adversary has been studied
extensively and is the best understood of all the adversaries
discussed in this paper. With full-information feedback, both the
\emph{Hedge} algorithm \cite{LW94,FS97} and the \emph{follow the
  perturbed leader (FPL)} algorithm \cite{Kalai:05} guarantee a
regret of $\Ocal(\sqrt{T})$, with a matching lower bound of
$\Omega(\sqrt{T})$ ---see, e.g., \cite{CesaBianchi:06}.  Analyses of
Hedge in settings where the loss range may vary over time have
also been considered ---see, e.g., \cite{CMS07}.
The oblivious setting with bandit feedback, where the player only
observes the incurred loss $f_t(X_{1:t})$, is called the
\emph{nonstochastic (or adversarial) multi-armed bandit} problem. In
this setting, the Exp3 algorithm of \cite{Auer:02}
guarantees the same regret $\Ocal(\sqrt{T})$ as the full-information
setting, and clearly the full-information lower bound
$\Omega(\sqrt{T})$ still applies.

The \emph{follow the lazy leader (FLL)} algorithm of \cite{Kalai:05}
is designed for the switching costs setting with full-information
feedback. The analysis of FLL guarantees that the oblivious component
of the player's expected regret (without counting the switching costs),
as well as the expected number of switches, is upper bounded by $\Ocal(\sqrt{T})$,
implying an expected regret of at most $\Ocal(\sqrt{T})$.

The algorithm of \cite{Merhav:02} focuses on the bounded memory
adversary with full-information feedback, referring to this problem as
\emph{loss functions with memory}, and guaranteeing a regret
of $\Ocal(T^{2/3})$. The work of \cite{Arora:12} extends this
result to the bandit feedback case, maintaining the same regret bound.

Learning with bandit feedback and switching costs has mostly been
considered in the economics literature, using a different setting than
ours and with prior knowledge assumptions (see \cite{Jun:04} for an
overview). The setting of stochastic oblivious adversaries (i.e.,
oblivious loss functions sampled i.i.d.\ from a fixed distribution)
was first studied by \cite{Agrawal:88}, where they show that
$\Ocal(\log T)$ switches are sufficient to asymptotically guarantee
logarithmic regret. The paper \cite{Ortner:10} achieves logarithmic regret
nonasymptotically with $\Ocal(\log T)$ switches.

Several other papers discuss online learning against ``adaptive''
adversaries \cite{Auer:02,Dani:06,Munos:10,McMahan:04}, but these
results are not relevant to our work and can be easily
misunderstood. For example, even the Exp3 algorithm of \cite{Auer:02}
has extensions to the ``adaptive'' adversary case, with a regret upper
bound of $\Ocal(\sqrt{T})$. This bound doesn't contradict the
$\Omega(T)$ lower bound for general adaptive adversaries mentioned
earlier, since these papers use the regret defined in
\eqref{eqn:oldRegret} rather than the regret used in our work, defined
in \eqref{eqn:policyRegret}.

Another related body of work lies in the field of competitive analysis
---see \cite{Borodin:98}, which also deals with loss functions
that depend on the player's past actions, and the adversary's memory may
even be unbounded. However, obtaining sublinear regret is generally
impossible in this case. Therefore, competitive analysis studies much
weaker performance metrics such as the competitive ratio, making it
orthogonal to our work.

%

\subsection{Our Contribution}

\begin{table*}
\centerline{
\begin{tabular}{|c||c|c|c|c|c|}
\hline
&oblivious&switching cost&memory of size 1&bounded memory&adaptive\\
\hline
\hline
\multicolumn{6}{|c|}{Full-Information Feedback}\\
\hline
\rule{0pt}{3ex} $\widetilde \Ocal$&$\sqrt{T}$&$\sqrt{T}$&$T^{2/3}$&$T^{2/3}$&$T$\\
\hline
$\Omega$&{$\sqrt{T}$}&{$\sqrt{T}$}&{$\sqrt{T}$}&{$\sqrt{T} ~\rightarrow~ \boldsymbol{T^{2/3}}$}&{$T$}\\
\hline
\hline
\multicolumn{6}{|c|}{Bandit Feedback}\\
\hline
\rule{0pt}{3ex} $\widetilde \Ocal$&$\sqrt{T}$&${T^{2/3}}$&${T^{2/3}}$&${T^{2/3}}$&$T$\\
\hline
$\Omega$&{$\sqrt{T}$}&{$\sqrt{T} ~\rightarrow~ \boldsymbol{T^{2/3}}$}&{$\sqrt{T} ~\rightarrow~ \boldsymbol{T^{2/3}}$}&{$\sqrt{T} ~\rightarrow~ \boldsymbol{T^{2/3}}$}&{$T$}\\
\hline
\end{tabular}
}
\caption{State-of-the-art upper and lower bounds on regret (as a function of $T$) against different adversary types. Our contribution to this table is presented in bold face.}
\label{tbl:bounds}
\end{table*}

In this paper, we make the following contributions  (see Table \ref{tbl:bounds}):
\begin{itemize}
    \item Our main technical contribution is a new lower bound on
      regret that matches the existing upper bounds in several of the
      settings discussed above. Specifically, our lower bound applies
      to the switching costs adversary with bandit feedback and
      to all strictly stronger adversaries.
\vspace{-1mm}
    \item Building on this lower bound, we prove another
      regret lower bound in the bounded memory setting with
      \emph{full-information} feedback, again matching the known upper bound.
\vspace{-1mm}
    \item We confirm that existing upper bounds on regret hold in our
      setting and match the lower bounds up to logarithmic factors.
\vspace{-1mm}
    \item Despite the lower bound, we show that for switching costs
      and bandit feedback, if we also assume \emph{stochastic
        i.i.d.\ losses}, then one can get a distribution-free regret
      bound of $\Ocal(\sqrt{T\log\log\log T})$ for finite action sets,
      with only $\mathcal{O}(\log\log T)$ switches. This result uses
      ideas from \cite{CesaBianchi:13}, and is deferred to
      \appref{app:stochastic}.
\end{itemize}
Our new lower bound is a significant step towards a complete
understanding of adaptive adversaries; observe that the upper and
lower bounds in Table \ref{tbl:bounds} essentially match in all but
one of the settings.

Our results have two important consequences. First, observe that the
optimal regret against the switching costs adversary is
$\Theta\bigl(\sqrt{T}\bigr)$ with full-information feedback, versus
$\Theta\bigl(T^{2/3}\bigr)$ with bandit feedback. To the best of our knowledge,
this is the first theoretical confirmation that learning with bandit
feedback is strictly harder than learning with full-information, even
on a small finite action set and even in terms of the dependence on
$T$ (previous gaps we are aware of were either in terms of the number
of actions \cite{Auer:02}, or required large or continuous action
spaces ---see, e.g., \cite{BuMuStoSz:11,Shamir:12}).  Moreover,
recall the regret bound of $\Ocal\bigl(\sqrt{T\log\log\log T}\bigr)$
against the stochastic i.i.d.\ adversary with switching costs and
bandit feedback. This demonstrates that dependencies in the loss
process must play a crucial role in controlling the power of the
switching costs adversary. Indeed, the $\Omega\bigl(T^{2/3}\bigr)$
lower bound proven in the next section heavily relies on such
dependencies.

Second, observe that in the full-information feedback case, the
optimal regret against a switching costs adversary is
$\Theta(\sqrt{T})$, whereas the optimal regret against the more
general bounded memory adversary is $\Omega(T^{2/3})$. This is
somewhat surprising given the ideas presented in \cite{Merhav:02} and
later extended in \cite{Arora:12}: The main technique used in these
papers is to take an algorithm originally designed for oblivious
adversaries, forcefully prevent it from switching actions very often,
and obtain a new algorithm that guarantees a regret of
$\Ocal(T^{2/3})$ against bounded memory adversaries. This would seem
to imply that a small number of switches is the key to dealing with
general bounded memory adversaries. Our result contradicts this
intuition by showing that controlling the number of switches is easier
then dealing with a general bounded memory adversary.

As noted above, our lower bounds require us to slightly weaken the standard
technical assumption that loss values lie in a fixed interval
$[0,C]$. We replace it with the following two assumptions:
\begin{enumerate}
\item \emph{Bounded range}. We assume that the loss values \emph{on each
  individual round} are bounded in an interval of constant size $C$, but we allow this interval to
  drift from round to round. Formally, $\forall t$, $\forall x_{1:t} \in \Acal^{t}$ and $\forall x'_{1:t} \in \Acal^{t}$,
\begin{equation}
\label{eqn:boundedRange}
\big| f_t(x_{1:t}) - f_t(x'_{1:t}) \big| ~\leq~ C~~.
\end{equation}
\item \emph{Bounded drift}. We also assume that the drift of each
  individual action from round to round is contained in a bounded
  interval of size $D_t$, where $D_t$ may grow slowly, as $\Ocal\bigl(\sqrt{\log(t)}\bigr)$. Formally, $\forall t$ and
  $\forall x_{1:t} \in \Acal^t$,
\begin{equation}
\label{eqn:boundedDrift}
\big| f_t(x_{1:t}) - f_{t+1}(x_{1:t}, x_t) \big| ~\leq~ D_t~~.
\end{equation}
\end{enumerate}
Since these assumptions are a relaxation of the standard assumption,
all of the known lower bounds on regret automatically extend to our
relaxed setting. For our results to be consistent with the current
state of the art, we must also prove that all of the known upper
bounds continue to hold after the relaxation, up to logarithmic
factors.

\section{Lower Bounds}

In this section, we prove lower bounds on the player's expected regret
in various settings.

\subsection{$\Omega(T^{2/3})$ with Switching Costs and Bandit Feedback}

We begin with a $\Omega(T^{2/3})$ regret lower bound against an
oblivious adversary with switching costs, when the player receives
bandit feedback. It is enough to consider a very simple setting, with
only two actions, labeled $1$ and $2$. Using the notation
introduced earlier, we use $\ell_1,\ell_2,\dots$ to denote the
oblivious sequence of loss functions chosen by the adversary before
adding the switching cost.
\begin{theorem}\label{thm:lowerbound}
For any
player strategy that relies on bandit feedback
and for any number of rounds
$T$, there exist loss functions $f_1,\dots,f_T$ that are oblivious
with switching costs, with a range
bounded by $C=2$, and a drift bounded by
$D_t = \sqrt{3\log(t)+16}$, such that $\E[R_T] \geq
\frac{1}{40}T^{2/3}$.
\end{theorem}
The full proof is given in \appref{app:lowerBound}, and here we give an informal proof sketch.  We
begin by constructing a \emph{randomized} adversarial strategy, where
the loss functions $\ell_1,\ldots,\ell_T$ are an instantiation of
random variables $L_t,\ldots,L_T$ defined as follows. Let
$\xi_1,\dots,\xi_T$ be i.i.d.\ standard Gaussian random variables (with
zero mean and unit variance) and let $Z$ be a random variable that
equals $-1$ or $1$ with equal probability. Using these random
variables, define for all $t = 1\ldots T$
\begin{align}
L_t(1) &~=~ \sum_{s=1}^t \xi_s ~~,\nonumber\\
L_t(2) &~=~ L_t(1) + Z T^{-1/3}~~.
\label{eqn:lossstrategy}
\end{align}
In words, $\{L_t(1)\}_{t=1}^T$ is simply a Gaussian random walk and
$\{L_t(2)\}_{t=1}^T$ is the same random walk, slightly shifted up or
down ---see figure \ref{fig:randomwalk} for an illustration.
It is straightforward to confirm that
this loss sequence has a bounded range, as required by the theorem: by
construction we have $|\ell_t(1) - \ell_t(2)| = T^{-1/3} \leq 1$ for all $t$,
and since the switching cost can add at most $1$ to the loss on each
round, we conclude that $|f_t(1) - f_t(2)| \leq 2$ for all $t$.
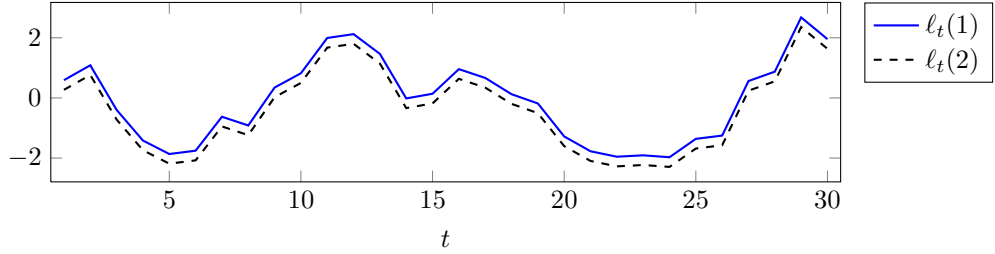
\begin{figure}
\centering
\begin{tikzpicture}
    \begin{axis}[
        xlabel=$t$,xmin=0.5, xmax=30.5, x=0.35cm, y=0.4cm, legend pos = outer north east]
    \addplot[blue,thick]
        plot coordinates {
(1,0.591478)
(2,1.085765)
(3,-0.397356)
(4,-1.417620)
(5,-1.864615)
(6,-1.754956)
(7,-0.626220)
(8,-0.916183)
(9,0.345368)
(10,0.820793)
(11,1.994909)
(12,2.121856)
(13,1.465040)
(14,-0.016359)
(15,0.139130)
(16,0.957682)
(17,0.665094)
(18,0.124307)
(19,-0.184335)
(20,-1.280928)
(21,-1.773938)
(22,-1.954677)
(23,-1.908836)
(24,-1.972619)
(25,-1.361284)
(26,-1.251966)
(27,0.562049)
(28,0.874073)
(29,2.678567)
(30,1.955445)
        };
    \addplot[black,dashed,thick] plot coordinates {
(1,0.269649)
(2,0.763936)
(3,-0.719185)
(4,-1.739450)
(5,-2.186445)
(6,-2.076786)
(7,-0.948050)
(8,-1.238013)
(9,0.023538)
(10,0.498963)
(11,1.673080)
(12,1.800027)
(13,1.143211)
(14,-0.338188)
(15,-0.182699)
(16,0.635852)
(17,0.343264)
(18,-0.197523)
(19,-0.506164)
(20,-1.602758)
(21,-2.095767)
(22,-2.276507)
(23,-2.230666)
(24,-2.294449)
(25,-1.683114)
(26,-1.573796)
(27,0.240219)
(28,0.552243)
(29,2.356737)
(30,1.633616)
    };
    \addlegendentry{$\ell_t(1)$}
    \addlegendentry{$\ell_t(2)$}
    \end{axis}
    \end{tikzpicture}
\caption{A particular realization of the random loss sequence
  defined in \eqref{eqn:lossstrategy}. The sequence of losses for action $1$
  follows a Gaussian random walk, whereas the sequence of losses for
  action $2$ follows the same random walk, but slightly shifted either
  up or down.}
\label{fig:randomwalk}
\end{figure}
Next, we show that the expected regret of any player against this
random loss sequence is $\Omega(T^{2/3})$, where expectation is taken
over the randomization of both the adversary and the player.  The
intuition is that the player can only gain information about which
action is better by switching between them. Otherwise, if he stays on
the same action, he only observes a random walk, and gets no further
information. Since the gap between the two losses on each round is
$T^{-1/3}$, the player must perform $\Omega(T^{2/3})$ switches before
he can identify the better action.  If the player performs that many
switches, the total regret incurred due to the switching costs is
$\Omega(T^{2/3})$. Alternatively, if the player performs $o(T^{2/3})$
switches, he can't identify the better action; as a result he suffers
an expected regret of $\Omega(T^{-1/3})$ on each round and a total
regret of $\Omega(T^{2/3})$.

Since the randomized loss sequence defined in
\eqref{eqn:lossstrategy}, plus a switching cost, achieves an expected
regret of $\Omega(T^{2/3})$, there must exist at least one
\emph{deterministic} loss sequence $\ell_1\ldots \ell_T$ with a regret
of $\Omega(T^{2/3})$. In our proof, we show that there exists such $\ell_1\ldots\ell_T$ with bounded drift.

\subsection{$\Omega(T^{2/3})$ with Bounded Memory and Full-Information Feedback}
We build on \thmref{thm:lowerbound} and prove a $\Omega(T^{2/3})$
regret lower bound in the \emph{full-information setting}, where we
get to see the entire loss vector on every round. To get this strong
result, we need to give the adversary a little bit of extra power:
memory of size $2$ instead of size $1$ as in the case of switching
costs. To show this result, we again consider a simple setting with two actions.

\begin{theorem}\label{thm:lowerboundfullinf}
For any
player strategy that relies on full-information feedback
and for any number of rounds
$T\geq 2$, there exist loss functions $f_1,\dots,f_T$, each with
a memory of size $m=2$, a range
bounded by $C=2$, and a drift bounded by
$D_t = \sqrt{3\log(t)+18}$, such that $\E[R_T] \geq \frac{1}{40}(T-1)^{2/3}$.
\end{theorem}
The formal proof is deferred to \appref{proof:lowerboundfullinf}
and
a proof sketch is given here.  The proof is based on a
reduction from full-information to bandit feedback that might be of
independent interest. We construct the adversarial loss sequence as
follows: on each round, the adversary assigns the \emph{same} loss to
both actions. Namely, the value of the loss depends only on the
player's previous two actions, and not on his action on the current
round. Recall that even in the full-information version of the game,
the player doesn't know what the losses would have been had he chosen
different actions in the past. Therefore, we have made the
full-information game as difficult as the bandit game.

Specifically, we construct an oblivious loss sequence $\ell_1\ldots
\ell_T$ as in \thmref{thm:lowerbound} and define
\begin{equation}
f_t(x_{1:t}) ~=~ \ell_{t-1}(x_{t-1}) + \Ind{x_{t-1} \neq x_{t-2}}~~.
\label{eqn:reductionFull}
\end{equation}
In words, we define the loss on round $t$ of the full-information game
to be equal to the loss on round $t-1$ of a
bandits-with-switching-costs game in which the player chooses the same
sequence of actions. This can be done with a memory of size $2$, since
the loss in \eqref{eqn:reductionFull} is fully specified by the
player's choices on rounds $t,t-1,t-2$. Therefore, the
$\Omega(T^{2/3})$ lower bound for switching costs and bandit feedback
extends to the full-information setting with a memory of size
at least $2$.

\section{Upper Bounds}
In this section, we show that the known upper bounds on regret, originally proved for bounded losses, can be extended to the case
of losses with bounded range and bounded drift. Specifically,
of the upper bounds that appear in Table \ref{tbl:bounds}, we
prove the following:
\begin{list}{\labelitemi}{\leftmargin=2em}
\item $\Ocal(\sqrt{T})$ for an oblivious adversary with
  switching costs, with full-information feedback.
\item $\widetilde{\Ocal}(\sqrt{T})$ for an oblivious adversary
  with bandit feedback (where $\widetilde{\Ocal}$ hides
  logarithmic factors).
\item $\widetilde{\Ocal}(T^{2/3})$ for a bounded memory
  adversary with bandit feedback.
\end{list}
The remaining upper bounds in Table \ref{tbl:bounds}
are either trivial or follow from the principle that an upper bound
still holds if we weaken the adversary or provide a more informative feedback.


\subsection{$\Ocal(\sqrt{T})$ with Switching Costs and Full-Information Feedback}

In this setting, $f_t(x_{1:t}) = \ell_t(x_t) + \Ind{x_t \neq
  x_{t-1}}$.  If the oblivious losses $\ell_1 \ldots \ell_T$ (without
the additional switching costs) were all bounded in $[0,1]$, the
\emph{Follow the Lazy Leader} (FLL) algorithm of \cite{Kalai:05}
would guarantee a regret of $\Ocal(\sqrt{T})$ with respect to these
losses (again, without the additional switching costs). Additionally,
FLL guarantees that its expected number of switches is
$\Ocal(\sqrt{T})$. We use a simple reduction to extend these
guarantees to loss functions with a range bounded in an interval of
size $C$ and with an arbitrary drift.

On round $t$, after choosing an action and receiving the loss function
$\ell_t$, the player defines the modified loss $\ell'_t(x) =
\frac{1}{C-1} \bigl( \ell_t(x) - \min_{y} \ell_t(y) \bigr)$ and feeds
it to the FLL algorithm. The FLL algorithm then chooses the next
action.
%
\begin{theorem} \label{thm:fullUpBound}
If each of the loss functions $f_1,f_2, \ldots$ is oblivious with
switching costs and has a range bounded by $C$ then the player
strategy described above attains $\Ocal(C\sqrt{T})$ expected regret.
\end{theorem}
The formal proof is given in \appref{app:upper} but the proof
technique is quite straightforward.  We first show that each $\ell'_t$
is bounded in $[0,1]$ and therefore the standard regret bound for FLL
holds with respect to the sequence of modified loss functions
$\ell'_1,\ell'_2,\ldots$. Then we show that the guarantees provided
for FLL imply a regret of $\Ocal(\sqrt{T})$ with respect to the
original loss sequence $f_1,f_2,\ldots$.

\subsection{$\widetilde{\Ocal}(\sqrt{T})$ with an Oblivious Adversary
  and Bandit Feedback}

In this setting, $f_t(x_{1:t})$ simply equals $\ell_t(x_t)$.
The reduction described in the previous subsection cannot be used in
the bandit setting, since $\min_x \ell_t(x)$ is unknown to the player,
and a different reduction is needed. The player sets a fixed horizon
$T$ and focuses on controlling his regret at time $T$; he can then use
a standard doubling trick \cite{CesaBianchi:06} to handle
an infinite horizon. The player uses the fact that each $f_t$ has a
range bounded by $C$. Additionally, he defines $D = \max_{t\leq T}
D_t$ and on each round he defines the modified loss
\begin{equation}\label{eqn:ftagUpper}
f'_t(x_{1:t})=
\frac{1}{2(C+D)}\big(\ell_t(x_t)-\ell_{t-1}(x_{t-1})\big)+\frac{1}{2}.
\end{equation}
Note that $f'_t(X_{1:t})$ can be computed by the player using only
bandit feedback.  The player then feeds $f'_t(X_{1:t})$ to an
algorithm that guarantees a $\Ocal(\sqrt{T})$ \emph{standard} regret
(see definition in \eqref{eqn:oldRegret}) against an adaptive
adversary. The Exp3.P algorithm, due to \cite{Auer:02}, is such an
algorithm. The player chooses his actions according to the choices
made by Exp3.P. The following theorem states that this reduction
results in a bandit algorithm that guarantees a regret of $\widetilde
\Ocal(\sqrt{T})$ against oblivious adversaries.
\begin{theorem} \label{thm:banditObliviousUpBound}
If each of the loss functions $f_1 \ldots f_T$ is oblivious with a range bounded by
$C$ and a drift bounded by $D_t = \Ocal\bigl(\sqrt{\log(t)}\bigr)$ then the
player strategy described above attains $\widetilde \Ocal(C\sqrt{T})$
expected regret.
\end{theorem}
%

The full proof is given in \appref{app:upper}.  In a
nutshell, we show that each $f_t'$ is an adaptive loss function
bounded in $[0,1]$ and therefore the analysis of Exp3.P guarantees a
regret of $\Ocal(\sqrt{T})$ with respect to the loss sequence
$f'_1\ldots f'_T$. Then, we show that this guarantee implies a regret
of $(C+D)\Ocal(\sqrt{T}) = \widetilde \Ocal(C\sqrt{T})$ with respect to
the original loss sequence $f_1\ldots f_T$.

\subsection{$\widetilde{\Ocal}(T^{2/3})$ with Bounded Memory and Bandit Feedback}

Proving an upper bound against an adversary with a memory of size
$m$, with bandit feedback, requires a more delicate reduction. As in
the previous section, we assume a finite horizon $T$ and we let $D =
\max_t D_t$. Let $K = |\Acal|$ be the number of actions available to the player.

Since $f_T(x_{1:t})$ depends only on the last $m+1$ actions in
$x_{1:t}$, we slightly overload our notation and define
$f_t(x_{t-m:t})$ to mean the same as $f_t(x_{1:t})$. To define the
reduction, the player fixes a base action $x_0 \in \Acal$ and for each
$t > m$ he defines the loss function
\[
    \fhat_t(x_{t-m:t}) = \frac{1}{2\bigl(C+(m+1)D\bigr)}\bigl(f_t(x_{t-m:t}) - f_{t-m-1}(x_0 \dots x_0) \bigr) + \frac{1}{2}~.
\]
Next, he divides the $T$ rounds into $J$ consecutive epochs of equal
length, where $J = \Theta(T^{2/3})$. We assume that the epoch length
$T/J$ is at least $2K(m+1)$, which is true when $T$ is sufficiently
large.
At the beginning of each epoch, the player plans his action sequence
for the entire epoch. He uses some of the rounds in the epoch for
exploration and the rest for exploitation. For each action in $\Acal$,
the player chooses an exploration interval of $2(m+1)$ consecutive
rounds within the epoch. These $K$ intervals are chosen randomly, but
they are not allowed to overlap, giving a total of $2K(m+1)$
exploration rounds in the epoch. The details of how these intervals are drawn appears in our analysis,
in \appref{app:upper}. The remaining $T/J - 2K(m+1)$ rounds are used
for exploitation.

The player runs the Hedge algorithm \cite{FS97} in the background,
invoking it only at the beginning of each epoch and using it to choose
one exploitation action that will be played consistently on all of the
exploitation rounds in the epoch.
In the exploration interval for action $x$, the player first plays
$m+1$ rounds of the base action $x_0$ followed by $m+1$ rounds of the
action $x$. Letting $t_x$ denote the first round in this interval, the
player uses the observed losses $f_{t_x+m}(x_0 \dots x_0)$ and
$f_{t_x+2m+1}(x \dots x)$ to compute $\fhat_{t_x+2m+1}(x \dots x)$. In
our analysis, we show that the latter is an unbiased estimate of the
average value of $\fhat_{t}(x \dots x)$ over $t$ in the epoch. At
the end of the epoch, the $K$ estimates are fed as feedback to the Hedge algorithm.

We prove the following regret bound, with the proof deferred to \appref{app:upper}.
\begin{theorem}
\label{thm:banditBoundedUpper}
If each of the loss functions $f_1 \ldots f_T$ is has a memory of size
$m$, a range bounded by $C$, and a drift bounded by $D_t =
\Ocal\bigl(\sqrt{\log(t)}\bigr)$ then the player strategy described above
attains $\widetilde \Ocal(T^{2/3})$ expected regret.
\end{theorem}

\section{Discussion}

In this paper, we studied the problem of prediction with expert advice
against different types of adversaries, ranging from the
oblivious adversary to the general adaptive adversary. We proved upper
and lower bounds on the player's regret against each of these
adversary types, in both the full-information and the bandit feedback
models. Our lower bounds essentially matched our upper bounds in all but one case: the adaptive adversary with a unit memory in the full-information setting, where we only know that regret is $\Omega(\sqrt{T})$ and $\Ocal(T^{2/3})$.

Our new bounds have two important consequences. First, we characterize
the regret attainable with switching costs, and show a setting where
predicting with bandit feedback is strictly more difficult than
predicting with full-information feedback ---even in terms of the
dependence on $T$, and even on small finite action sets. Second, in
the full-information setting, we show that predicting against a
switching costs adversary is strictly easier than predicting against an
arbitrary adversary with a bounded memory.

To obtain our results, we had to slightly relax the standard assumption that loss values are bounded in $[0,1]$. Re-introducing this assumption and proving similar lower bounds remains an elusive open problem.
Many other questions remain unanswered. Can we
characterize the dependence of the regret on the size of the action set $\Acal$?
Can we strengthen any of our expected regret bounds to bounds
that hold with high probability? Can any of our results be generalized
to more sophisticated notions of regret, such as shifting regret and
swap regret, as in \cite{Arora:12}?

In addition to the adversary types discussed in this paper, there are
other interesting classes of adversaries that lie between the
oblivious and the adaptive. One of these is the oblivious adversary
with delayed feedback, briefly mentioned in the introduction. While
some results for this adversary exist (see \cite{Mesterharm:05}), the
attainable regret, especially in the bandit feedback case, is not
clear. Another interesting case is the family of
\emph{deterministically adaptive} adversaries, which includes
adversaries that adapt to the player's actions (so they are not
oblivious) in a known deterministic way, rather than in a secret
malicious way. For example, imagine playing a multi-armed bandit game
where the loss values are initially oblivious, but whenever the player
chooses an arm with zero loss, the loss of the same arm on the next
round is deterministically changed to zero. In other words, whenever
the player suffers a zero loss, he knows that choosing the same arm
again guarantees another zero loss. This is an important setting
because many real-world online prediction scenarios are indeed
deterministically adaptive.

\bibliographystyle{plainnat}
\bibliography{bib}


\newpage
\appendix

\noindent {\huge \bf Appendices}
\renewcommand{\argmin}{\mathop{\rm argmin}}
\renewcommand{\argmax}{\mathop{\rm argmax}}
\newcommand{\muhat}{\widehat{\mu}}
\newcommand{\khat}{\widehat{x}}

\section{Distribution-free regret bound for bandits with switching costs}\label{app:stochastic}
In this appendix we adapt results of \cite{CesaBianchi:13} to show a strategy that achieves $\mathcal{O}\bigl(\sqrt{T\log\log\log T}\bigr)$ regret against any i.i.d.\ oblivious adversary in the bandit setting with switching costs, assuming a finite action set $\Acal=\{1\ldots K\}$. The strategy used by this stochastic adversary is specified by a probability distribution over oblivious loss functions. The oblivious loss function for each step $t=1,2,\dots$ is the realization on an independent draw $L_t$ from this distribution. The regret of a player choosing actions $X_0=X_1,X_2,\dots$ is defined by
\[
    R_T ~=~ \sum_{t=1}^T \E_t\bigl[L_t(X_t) + \Ind{X_t \neq X_{t-1}} \bigr] - \min_{x \in \Acal} \sum_{t=1}^T \E\bigl[L_t(x)\bigr]
\]
where the expectation $\E$ is over the random draw of each $L_t$ and the possible randomization of the player, and the expectation $\E_t$ is conditioned over $X_1,L_1(X_1),\dots,X_{t-1},L_{t-1}(X_{t-1})$.

Our result focuses on loss distributions such that the law of each marginal $L_1(x)$ is subgaussian. A random variable $Z$ is subgaussian if there exist constants $b,c$ such that for any $a > 0$
$
    \Pr\bigl(Z > \E\,Z + a\bigr) \le b e^{-ca^2}
$
and
$
    \Pr\bigl(Z < \E\,Z -a\bigr) \le b e^{-ca^2}
$.
One can then show that, for any i.i.d.\ sequence $Z_1,\dots,Z_T$ of subgaussian random variables,
\begin{equation}
\label{eq:azuma}
    \Pr\left( \left|\frac{1}{T}\sum_{t=1}^T Z_t - \E\,Z_1 \right| > \sqrt{\frac{112b}{cT}\ln\frac{1}{\delta}} \right) \le \delta~.
\end{equation}
In the following, we use the notation $\E\bigl[L_t(x)\bigr] = \mu(x)$ and
${\displaystyle
    \mu^* = \min_{x\in\Acal} \mu(x)~.
}$
\begin{theorem}
Consider a finite action set $\Acal=\{1\ldots K\}$. Then for each $T$ there exists a deterministic player strategy for the bandit game with i.i.d.\ oblivious adversaries and switching costs, whose regret after $T$ steps is $\mathcal{O}\bigl(\sqrt{T\log\log\log T}\bigr)$ with high probability, provided the distribution of $L_1(x)$ is sugaussian for each $x \in \Acal$.
\end{theorem}
\begin{proof}
Consider the following player that proceeds in stages. At each stage $s=1,2,\dots,S$, the player maintains a set $A_s \subseteq \Acal$ of active actions. Each action is played $T_s/|A_s|$ times in a round-robin fashion, where $T_s = T^{1-2^{-s}}$ is the total number of plays in stage $s$ and $T$ is the known horizon. Note that the overall number of switches is at most $KS$, where
\[
    S = \min\left\{j \in \naturals \,:\, \sum_{s=1}^j T_s \ge T \right\} = \mathcal{O}\bigl(\ln\ln T\bigr)~.
\]
Let $\muhat_s(x)$ the sample mean of losses for action $x$ in stage $s$, and define
\[
    \khat_s = \argmin_{x \in A_s} \muhat_s(x)
\]
the best empirical action in stage $s$. The sets $A_s$ of active actions are defined as follows: $A_1 = \Acal$ and
\[
    A_s = \Bigl\{ x \in A_{i-1} \,:\, \muhat_{s-1}(x) \le \muhat_{s-1}(\khat_{s-1}) + 2C_{s-1} \Bigr\}
\]
where
\[
    C_s = \sqrt{112(b/c)\frac{K}{T_s}\ln\frac{KS}{\delta}}~.
\]
Note that $A_S \subseteq \cdots \subseteq A_1$ by construction. Also, using~(\ref{eq:azuma}) and the union bound we have that
\begin{equation}
\label{eq:uniform}
    \max_{x \in A_s} \bigl| \muhat_s(x) - \mu(x) \bigr| \le C_s
\end{equation}
simultaneously for all $s=1,\dots,S$ with probability at least $1-\delta$.

We claim the following.
\begin{claim}
With probability at least $1-\delta$,
\begin{align*}
    x^* \in \bigcap_{s=1}^S A_s
\qquad\text{and}\qquad
    0 \le \muhat_s(x^*) - \muhat_s(\khat_s) \le 2 C_s \quad \text{for all $s=1,\dots,S$.}
\end{align*}
\end{claim}
\noindent\textit{Proof of Claim.}
We prove the lemma by induction on $s=1,\dots,S$. We first show that the base case $s=1$ holds with probability at least $1-\delta/S$. Then we show that if the claim holds for $s-1$, then it holds for $s$ with probability at least $1-\delta/S$ over all random events in stage $s$. Therefore, using a union bound over $s=1,\dots,S$ we get that the claim holds simultaneously for all $s$ with probability at least $1-\delta$.

For the base case $s=1$ note that $x^* \in A_1$ by definition, and thus $\muhat_1(\khat_1) \le \muhat_1(x^*)$ holds. Moreover, using~(\ref{eq:uniform}) we obtain that
\[
    \muhat_1(x^*) - \mu(x^*) \le C_1
\qquad\text{and}\qquad
    \mu(\khat_1) - \muhat_1(\khat_1) \le C_1
\]
holds with probability at least $1-\delta/S$. Since $\mu(x^*) - \mu(\khat_1) \le 0$ by definition of $x^*$, we obtain
\[
    0
\le
    \muhat_1(x^*) - \muhat_1(\khat_1)
\le
    2C_1
\]
as required. We now prove the claim for $s > 1$ using the inductive assumption
\[
    x^* \in A_{s-1}
\qquad\text{and}\qquad
    0
\le
    \muhat_{s-1}(x^*) - \muhat_{s-1}(\khat_{s-1})
\le
    2C_{s-1}~.
\]
The inductive assumption directly implies that $x^*\in A_s$. Thus we have $\muhat_i(\khat_s) \le \muhat_s(x^*)$, because $\khat_s$ minimizes $\muhat_s$ over a set that contains $x^*$. The rest of the proof of the claim closely follows that of the base case $s=1$. \hfill $\Box$

\medskip\noindent
Now, for any $s=1,\dots,S$ and for any $x \in A_s$ we have that
\begin{align*}
    \mu(x) - \mu(x^*)
&\le
    \muhat_{s-1}(x) - \mu(x^*) + C_{s-1} \qquad \text{by~(\ref{eq:uniform})}
\\&\le
    \muhat_{s-1}(\khat_{s-1}) - \mu(x^*) + 3C_{s-1} \qquad \text{by definition of $A_{s-1}$, since $x \in A_s \subseteq A_{s-1}$}
\\&\le
    \muhat_{s-1}(x^*) - \mu(x^*) + 3C_{s-1} \qquad \text{since $\khat_{s-1}$ minimizes $\muhat_{s-1}$ in $A_{s-1}$}
\\&\le
    4C_{s-1} \qquad \text{by~(\ref{eq:uniform})}
\end{align*}
holds with probability at least $1-\delta/S$. Hence, recalling that
\[
    \sum_{t=1}^T \Ind{X_t \neq X_{t-1}} \le KS
\]
holds deterministically, the regret of the player over the $T$ plays can be bounded as follows
\begin{align*}
    KS + \sum_{t=1}^T \bigl( \mu(X_t) - \mu^* \bigr)
&=
    KS + \sum_{s=1}^S \frac{T_s}{|A_s|} \sum_{x \in A_s} \bigl( \mu(x) - \mu^* \bigr)
\\&=
    KS + \frac{T_1}{K} \sum_{i=1}^K \bigl( \mu(x) - \mu^* \bigr) + \sum_{s=2}^S \frac{T_s}{|A_s|} \sum_{x \in A_s} \bigl( \mu(x) - \mu^* \bigr)
\\&\le
    KS + T_1\mu^* + \sum_{i=2}^S 4T_s\sqrt{112(b/c)\frac{K}{T_s}\ln\frac{KS}{\delta}}
\\&=
    KS + T_1\mu^* + 4\sqrt{112(b/c)K\ln\frac{KS}{\delta}}\sum_{s=2}^S \frac{T_s}{\sqrt{T_{s-1}}}
\end{align*}
Now, since $T_1 = \sqrt{T}$, $T_s/\sqrt{T_{s-1}} = \sqrt{T}$ and $S = \mathcal{O}\bigl(\ln\ln T\bigr)$, we obtain that with probability at least $1-\delta$ the regret is at most of order
\[
K\ln\ln T + \mu^*\sqrt{T} + \sqrt{KT\left(\ln\frac{K}{\delta} + \ln\ln\ln T\right) }
\]
as desired.
\end{proof}

\section{Proof of \thmref{thm:lowerbound}}
\label{app:lowerBound}

As mentioned in the text, we first consider the player's expected
regret against a randomized adversary. Specifically, we define
$$
\forall t~~~L_t(1) ~=~ \sum_{s=1}^t \xi_s ~~~\text{and}~~~
L_t(2) ~=~ L_t(1) + Z \epsilon~~,
$$
where $\xi_1\ldots \xi_T$ are independent standard Gaussians, $Z$
equals $-1$ or $1$ with equal probability, and $\epsilon$ is the gap
between the losses of the two actions (which will later be set to
$\epsilon = T^{-1/3}$).

Next, we assume for now, without loss of generality, that the player is
deterministic.  A deterministic player chooses each action $X_t$ as a
deterministic function of the random losses suffered on the previous
rounds, $L_1(X_1)\dots L_{t-1}(X_{t-1})$.  We can make this assumption
because any randomized player strategy can be seen as a distribution
over deterministic player strategies, and since the randomization used
by the adversary is independent of the player's strategy.

In the results below, $\Pr$ denotes the distribution of the randomized
adversary.  We also introduce the conditional distributions $\Sm =
\Pr(\cdot\mid Z>0)$ (i.e., $1$ is the better action) and $\Qm =
\Pr(\cdot\mid Z<0)$ (i.e., $2$ is the better action). Since $Z$ has an
equal probability of being negative or positive, it holds that $\Pr =
\frac{1}{2}(\Sm + \Qm)$.

We begin with the following technical lemma.
\begin{lemma}\label{lem:1}
Let $\Ind{x_{t-1}\neq x_t}$ indicate whether the player switched actions on round $t$ (and $1$ for $t=1$). Then for any event $A$,
\[
\bigl|\Sm(A)-\Qm(A)\bigr|\leq \epsilon\sqrt{\E\left[\sum_{t=1}^{T}\Ind{X_t\neq X_{t-1}}\right]}
\]
where the expectation in the right-hand side is with respect to $\Pr$.
\end{lemma}

\begin{proof}
To show this, we use the chain rule for relative entropy, which implies
\begin{equation}\label{eq:chainn}
\KL\bigl(\Sm ~\big\|~ \Qm\bigr) = \sum_{t=1}^{T} \KL\bigl(\Sm_{t-1} ~\big\|~ \Qm_{t-1} \bigr)
\end{equation}
where $\Sm_{t-1}$ and $\Qm_{t-1}$ denote the distributions of the player's loss $L_t(x_t)$ conditioned on $L_1,\dots,L_{t-1}$, when the joint distribution of $L_1,\dots,L_T$ is, respectively, $\Sm$ and $\Qm$.

Let us focus on a particular term $\KL\bigl(\Sm_{t-1}~\big\|~\Qm_{t-1}\bigr)$ and a particular realization of the random losses $L_1,\dots,L_{t-1}$. Since we assume a deterministic player strategy, for any such realization the player's choices $x_{1:t}$ are all determined, and we deterministically have that the player either switched or not at time $t$. If he did not switch, then $L_t(x_t)$ is distributed as $L_{t-1}(x_{t-1})+\xi_t$ under both measures $\Sm_{t-1}$ and $\Qm_{t-1}$, so the relative entropy between them is zero. If he did switch, then $L_t(x_t)$ is distributed as $L_{t-1}(x_{t-1})-\epsilon+\xi$ under $\Sm_{t-1}$ (where the switch is towards the best action), and as $L_{t-1}(x_{t-1})+\epsilon+\xi$ under $\Qm_{t-1}$ (where the switch is towards the worst action). Hence, the relative entropy is the same as two standard Gaussians whose means are shifted by $2\epsilon$, namely $2\epsilon^2$. So overall, we can upper bound \eqref{eq:chainn} by
\begin{equation}
\label{eq:kl-bound}
    2\epsilon^2\,\E\left[\sum_{t=1}^T \Ind{X_t\neq X_{t-1}} \,\bigg|\, Z > 0 \right]~.
\end{equation}
Using a similar argument, we also show that $\KL\bigl(\Qm ~\big\|~ \Sm\bigr)$ is upper bounded by \eqref{eq:kl-bound} in which the conditioning on $Z > 0$ is replaced by $Z < 0$. Then, Pinsker's inequality implies that $\bigl|\Sm(A)-\Qm(A)\bigr|^2$ is at most
\[
\frac{\epsilon^2}{2}\left(\E\left[\sum_{t=1}^T \Ind{X_t\neq X_{t-1}} \,\bigg|\, Z > 0 \right] + \E\left[\sum_{t=1}^T \Ind{X_t\neq X_{t-1}} \,\bigg|\, Z < 0 \right]\right) =  \epsilon^2\E\left[\sum_{t=1}^T \Ind{X_t\neq X_{t-1}} \right]
\]
which gives the desired bound.
\end{proof}

With this lemma, we can prove a lower bound on the expected regret for randomized adversaries.
\begin{lemma}\label{lem:ll}
By picking $\epsilon=T^{-1/3}$, the expected
regret of any deterministic player strategy, over the randomness of
the adversary, is at least $\frac{1}{10} T^{2/3}$.
\end{lemma}
\begin{proof}
Let $A$ be the event that the worst action (action $2$ if
$Z>0$, and $1$ if $Z<0$) was picked by the player at
least $T/2$ times. Also, let $S_T = \sum_{t=1}^{T}\Ind{X_t\neq
  X_{t-1}}$ be the number of switches the player performs. Then
\[
\E[R_T] \ge \E\left[\max\left\{S_T, \frac{\epsilon T}{2}\Ind{A}\right\}\right]
\ge \E\left[\frac{1}{2}\left(S_T+\frac{\epsilon T}{2}\Ind{A}\right)\right]
= \frac{1}{2}\E[S_T]+\frac{\epsilon T}{4}\Pr(A)~.
\]
Moreover, letting $A_1$ denote the event that the player chose action
$1$ at least $T/2$ times, and letting $A_2$ denote the event that the
player chose action $2$ at least $T/2$ times, we have $\Pr(A) =
\frac{1}{2}\bigl(\Sm(A_2)+\Qm(A_1)\bigr)$. Substituting this, we get
\[
\frac{1}{2}\E[S_T]+\frac{\epsilon T}{8}\bigl(\Sm(A_2)+\Qm(A_1)\bigr).
\]
Using Lemma~\ref{lem:1} to lower bound $\Qm(A_1)$ via $\Sm(A_1)$, we get a lower bound of
\begin{align*}
&\frac{1}{2}\E[S_T]+\frac{\epsilon T}{8}\left(\Sm(A_2)+\Sm(A_1)-\epsilon \sqrt{\E[S_T]}\right)
~\geq~ \frac{1}{2}\E[S_T]+\frac{\epsilon T}{8}\left(\Sm(A_1 \cup A_2)-\epsilon \sqrt{\E[S_T]}\right)\\
&~=~ \frac{1}{2}\E[S_T]+\frac{\epsilon T}{8}\left(1-\epsilon \sqrt{\E[S_T]}\right)
~=~ \frac{1}{2}\E[S_T]-\frac{\epsilon^2 T}{8}\sqrt{\E[S_T]}+\frac{\epsilon T}{8},
\end{align*}
where we used a union bound and the fact that either $A_1$ or $A_2$
always holds. This is a quadratic function of $\sqrt{\E[S_T]}$, and it
is easily verified that the lowest possible value it can attain (for
any value of $\E[S_T]$) is
\[
\frac{\epsilon T}{8}-\frac{\epsilon^4 T^2}{128}.
\]
Picking $\epsilon = T^{-1/3}$, this equals
$\left(\frac{1}{8}-\frac{1}{128}\right)T^{2/3} > \frac{1}{10}
T^{2/3}$.
\end{proof}
The lemma above tells us that for the randomized adversary strategy we
have devised, the expected regret for any deterministic player is at
least $\frac{1}{10} T^{2/3}$. This implies that there exist some
\emph{deterministic} adversarial strategy, for which the expected
regret of any \emph{possibly randomized} player is at least
$\frac{1}{10} T^{2/3}$. However, we are not done yet, since this
strategy doesn't guarantee that the losses have bounded drift: In our
case, the variation is governed by a potentially unbounded Gaussian
random variable, so the deterministic adversary strategy that we
picked might have an arbitrarily large drift. So now, our goal will
be to show that there exists some deterministic adversarial strategy
for which the expected regret is large, \emph{and the variation is bounded}.
To do this, the plan is to show that the probabilities (over the
adversary's strategy) of the two events are large, summing to a number
larger than one. This means there is some realization of the losses such
that both events occur. We first state and prove two auxiliary lemmas,
and then provide two more fundamental lemmas which together give us the required
result.

\begin{lemma}\label{lem:randlow}
Let $Y$ be a random variable in $[-b,b]$ (where $b>0$), and $\E[Y]\geq c$ for some $c\in [0,b/2]$. Then we have
\[
\Pr\left(Y\geq c/2\right) \geq \frac{c}{2b-c} \geq \frac{c}{2b} .
\]
\end{lemma}
\begin{proof}
\begin{align*}
c \leq \E[Y] &= \Pr(Y\geq c/2)\E[Y \mid Y\geq c/2]+\Pr(Y<c/2)\E[Y \mid Y<c/2]
\\ & \leq
    \Pr(Y\geq c/2)b+\bigl(1-\Pr(Y\geq c/2)\bigr)c/2
\end{align*}
Solving for $\Pr(Y\geq c/2)$ gives the desired result.
\end{proof}

\begin{lemma}\label{lem:gaussunif}
Let $\xi_1,\xi_2,\ldots$ be an infinite sequence of independent standard
Gaussian random variables. Then for any $\delta\in (0,1)$
\[
\Pr\left(\exists t~:~|\xi_t|\geq \sqrt{3\log(2t/\delta)}\right) \leq \delta.
\]
\end{lemma}
\begin{proof}
By a standard Gaussian tail bound, we have that $\Pr(|\xi_t|>x) \leq \exp(-x^2/2)$
for any $x\geq 0$. This implies that
\[
\Pr(|\xi_t|\geq \sqrt{3\log(2t/\delta)}) \leq \left(\frac{\delta}{2t}\right)^{3/2}.
\]
By a union bound, we get that
\[
\Pr\left(\exists t~:~|\xi_t|\geq \sqrt{3\log(2t/\delta)}\right) \leq \sum_{t=1}^{\infty}
\left(\frac{\delta}{2t}\right)^{3/2} \leq \delta^{3/2} < \delta.
\]
\end{proof}

\begin{lemma}\label{lem:problowbound}
For any (possibly randomized) player strategy,
it holds that
\[
\Pr\left(\E_{\text{player}}[R_T]\geq \frac{1}{40}T^{2/3}\right) \geq \frac{1}{40},
\]
where $\Pr$ is over the adversary's randomization, and
$\E_{\text{player}}[R_T]$ is the player's expected regret
(over the player's randomization).
\end{lemma}
\begin{proof}
By \lemref{lem:ll}, we already know that
\begin{equation}\label{eq:randlow}
\E\bigl[\E_{\text{player}}[R_T]\bigr]\ge\frac{1}{10}T^{2/3},
\end{equation}
since if we have a $T^{2/3}/10$ lower bound on the regret for any deterministic
player strategy, the same holds for any randomized player strategy.
Our approach is to apply \lemref{lem:randlow} in order to convert this into a probability
lower bound as in the lemma statement. However, we cannot apply \lemref{lem:randlow} as-is,
since $\E_{\text{player}}[R_T]$ can be as large as $\Omega(T)$, and the resulting
bound is too weak. Instead, we show that there exists a \emph{different} player strategy, with expected
regret $\E_{\widetilde{\text{player}}}[R_T]$, such that
$|\E_{\widetilde{\text{player}}}[R_T]|$ is always at most $2T^{2/3}$ and
\begin{equation}\label{eq:expineq}
\E_{\widetilde{\text{player}}}[R_T] \leq 2~\E_{\text{player}}[R_T]
\end{equation}
for any realization of the adversary's random strategy. Also, analogous to \eqref{eq:randlow}, we have
$\E[\E_{\widetilde{\text{player}}}[R_T]] \geq \frac{1}{10}T^{2/3}$ by \lemref{lem:ll}.
Therefore, using \eqref{eq:expineq} and \lemref{lem:randlow}, we get that
\[
\Pr\left(\E_{\text{player}}[R_T]\geq \frac{1}{40}T^{2/3}\right)
\geq \Pr\left(\E_{\widetilde{\text{player}}}[R_T]\geq \frac{1}{20}T^{2/3}\right)
\geq \frac{1}{40}
\]
as required.

The new player strategy we consider depends on the horizon $T$, and is very simple: It is
identical to the original player strategy, but whenever the number of action switches
reaches $\lfloor T^{2/3} \rfloor$, the player ``freezes'' in its current action, and
keeps playing the same action till $T$ rounds are elapsed. Clearly, the number of switches
with this strategy can never be more than $T^{2/3}$, and since the regret in terms of the loss $\ell_t$
at each round is either $0$ or $T^{-1/3}$, we get that the total regret $R_T$ can never be more than
$T^{2/3}+T*T^{-1/3}=2T^{2/3}$.

To prove \eqref{eq:expineq}, we consider some instantiation of the adversary's random strategy,
and note that for any
realization of the player's random coin tosses, the regret can only differ between the
two strategies if $S_T$ (the total number of switches) is at least $\lfloor T^{2/3} \rfloor$.
Therefore, we have
$\Pr_{\text{player}}\left(S_T < \lfloor T^{2/3} \rfloor\right) = \Pr_{\widetilde{\text{player}}}\left(S_T < \lfloor T^{2/3}\rfloor\right)$, $\Pr_{\text{player}}\left(S_T \geq \lfloor T^{2/3} \rfloor\right) = \Pr_{\widetilde{\text{player}}}\left(S_T \geq \lfloor T^{2/3}\rfloor\right)$
and $\E_{\text{player}}[R_T|S_T < \lfloor T^{2/3}\rfloor ]=\E_{\widetilde{\text{player}}}[R_T|S_T < \lfloor T^{2/3}\rfloor]$.
Also, we recall that $R_T \geq 0$ with the adversary strategy that we consider (since one action is
always worse than the other action at all rounds). Finally, we note that
if $S_T\geq \lfloor T^{2/3}\rfloor$, then the regret for both
strategies is at least $\lfloor T^{2/3} \rfloor$ (since with the adversary strategy that
we consider, the number of switches is a lower bound on the regret).
Using these observations, we have
\begin{align*}
&\E_{\widetilde{\text{player}}}[R_T]\\
&= \Pr_{\widetilde{\text{player}}}(S_T < \lfloor T^{2/3}\rfloor)\E_{\widetilde{\text{player}}}[R_T|S_T < \lfloor T^{2/3}\rfloor ]
+ \Pr_{\widetilde{\text{player}}}(S_T \geq \lfloor T^{2/3}\rfloor)\E_{\widetilde{\text{player}}}[R_T|S_T \geq \lfloor T^{2/3}\rfloor ]\\
&\leq \Pr_{\widetilde{\text{player}}}(S_T < \lfloor T^{2/3}\rfloor)\E_{\widetilde{\text{player}}}[R_T|S_T < \lfloor T^{2/3}\rfloor ]
+ \Pr_{\widetilde{\text{player}}}(S_T \geq \lfloor T^{2/3}\rfloor)2T^{2/3}\\
&= \Pr_{\text{player}}(S_T < \lfloor T^{2/3}\rfloor)\E_{\text{player}}[R_T|S_T < \lfloor T^{2/3}\rfloor ]
+ \Pr_{\text{player}}(S_T \geq \lfloor T^{2/3}\rfloor)2T^{2/3}\\
&\leq 2\left( \Pr_{\text{player}}(S_T < \lfloor T^{2/3}\rfloor)\E_{\text{player}}[R_T|S_T < \lfloor T^{2/3}\rfloor ]
+ \Pr_{\text{player}}(S_T \geq \lfloor T^{2/3}\rfloor)T^{2/3}\right)\\
&\leq 2\left(\Pr_{\text{player}}(S_T < \lfloor T^{2/3}\rfloor)\E_{\text{player}}[R_T|S_T < \lfloor T^{2/3}\rfloor ]
+ \Pr_{\text{player}}(S_T \geq \lfloor T^{2/3}\rfloor)\E_{\text{player}}[R_T|S_T \geq \lfloor T^{2/3}\rfloor ]\right)\\
&= 2~\E_{\text{player}}[R_T],
\end{align*}
where in the second-to-last step we used the fact that if $S_T \geq \lfloor T^{2/3}\rfloor$, then the regret is at least $\lfloor T^{2/3}\rfloor$, plus we must have picked the worst action (worst by $T^{-1/3}$ than the best action) at least $\Omega(T^{2/3})$ times, hence the total regret is certainly at least $T^{2/3}$.
\end{proof}

Finally, we use \lemref{lem:gaussunif} with $\delta=1/80$, to get that with
probability at least $1-1/80$, the drift factor $D_t$ of the adversarial strategy
is at most $\sqrt{3\log(160t)} \leq \sqrt{3\log(t)+16}$ for all $t$. Moreover,
\lemref{lem:problowbound} tells us that $\E_{\text{player}}[R_T]$ is at least
$\frac{1}{40}T^{2/3}$ with probability at least $1/40$. This implies that the
intersection of the two events is non-empty, and there exists some deterministic adversarial
strategy, such that the drift $D_t \leq \sqrt{3\log(t)+16}$ for all $t$,
\emph{and} the expected regret is at least $\frac{1}{40}T^{2/3}$ as required.

\section{Proof of \thmref{thm:lowerboundfullinf}}\label{proof:lowerboundfullinf}
\thmref{thm:lowerbound} guarantees that given any player's strategy, there is some deterministic adversary strategy with a lower bound on the regret. However, as part of proving  \thmref{thm:lowerbound}, we actually showed that there exists some \emph{randomized} adversary strategy $\{\hat{f}_t\}_{t=1}^{T}$ with memory size $1$, such that for \emph{any} (possibly randomized) player strategy $x_{1:t}$,
\begin{equation}\label{eq:expbound}
\E\left[\sum_{t=1}^{T}\hat{f}_{t}(X_{t-1},X_{t})-\min_{x\in \Acal} \sum_{t=1}^{T}\hat{f}_{t}(x,x)\right] \geq \frac{1}{10}T^{2/3}
\end{equation}
(see \lemref{lem:ll}). We now use this strategy to define a randomized adversary strategy for our setting (with memory size $2$), for a game of $T+1$ rounds. We let $f_1(x_1)=0$ for any $x_1$, $f_2(x_1,x_2)=\hat{f}_1(x_1)$, and for every $t=3\ldots T+1$,
\begin{equation}\label{eq:trans}
f_{t}(x_{t-2},x_{t-1},x_{t})
= \hat{f}_{t-1}(x_{t-2},x_{t-1})~.
\end{equation}
Now, suppose we had some (possibly randomized) player strategy $X_1\ldots X_{T+1}$, so that in expectation over the player and adversary strategies, we have
\[
\E\left[\sum_{t=1}^{T+1}f_t(X_{t-2},X_{t-1},X_{t})-\min_{x\in \Acal} \sum_{t=1}^{T+1}f_t(x,x,x)\right] < \frac{1}{10}T^{2/3}.
\]
In particular, since $f_1$ is always $0$, it would imply that
\[
\E\left[\sum_{t=2}^{T+1}f_t(X_{t-2},X_{t-1},X_{t})-\min_{x\in \Acal} \sum_{t=2}^{T+1}f_t(x,x,x)\right] < \frac{1}{10}T^{2/3}~.
\]
By \eqref{eq:trans}, this implies
\[
\E\left[\sum_{t=1}^{T}\hat{f}_{t}(X_{t-1},X_{t})-\min_{x\in \Acal} \sum_{t=1}^{T}\hat{f}_{t}(x,x)\right] < \frac{1}{10}T^{2/3}~.
\]
Thus, if we could implement the player strategy $X_1\ldots X_{T}$ in the bandits-with-switching-costs setting, it will contradict \eqref{eq:expbound}. To see that this indeed can happen, note that each $X_{t}$ is a (possibly randomized) function of $X_{1:t-1}$ as well as $\{f_\tau(X_{\tau-2},X_{\tau-1},X_{\tau})\}_{\tau=1}^{t-1}$. But again, due to \eqref{eq:trans} and the fact that $f_1$ is always $0$, $X_t$ can in fact be defined using $X_{1:t-1}$ and
\[
\bigl\{f_\tau(X_{\tau-2},X_{\tau-1},X_{\tau})\bigr\}_{\tau=2}^{t-1}
~=~
\bigl\{\hat{f}_{\tau-1}(X_{\tau-2},X_{\tau-1})\bigr\}_{\tau=2}^{t-1}~.
\]
The right hand side is an observable quantity in the bandit setting: In each round $t$, we know what are the set of losses  $\{\hat{f}_{\tau-1}(X_{\tau-2},X_{\tau-1})\}_{\tau=2}^{t-1}$ that we obtained. Thus, we can simulate the strategy $x_{1:t}$ in the bandit-with-switching-costs setting, and get an expected regret smaller than $\frac{1}{10}T^{2/3}$, contradicting \eqref{eq:expbound}. Thus, the expected regret (for a game of $T+1$ rounds) must be at least $\frac{1}{10}T^{2/3}$. Substituting $T$ instead of $T+1$, we get that the expected regret for a game with $T$ rounds is at least $\frac{1}{10}(T-1)^{2/3}$.

The regret bound we just now obtained is in expectation over the randomized adversary strategy, and holds for any player's strategy. We now use the same line of argument as in the last part of \thmref{thm:lowerbound}'s proof, to show that for any (possibly randomized) player's strategy, there exists some \emph{deterministic} adversary strategy, with a similar expected regret bound, and with losses of bounded drift. Specifically, a result completely analogous to \lemref{lem:problowbound} implies that
\[
\Pr\left(\E_{\text{player}}[R_T] \geq \frac{1}{40}~(T-1)^{2/3}\right) \geq \frac{1}{40}\left(\frac{T-1}{T}\right)^{2/3},
\]
which is at least $1/80$ for any $T>1$ (if $T=1$ the bound in the theorem is trivial from the non-negativity of $R_T$ for the
adversary strategy that we consider). Moreover, using \eqref{lem:gaussunif} as in the proof of \thmref{thm:lowerbound}, the probability of the loss drift being at most $\sqrt{3\log(320t)} \leq \sqrt{3\log(t)+18}$ is at least $1-1/160$. Thus, the intersection of the two events is not empty, and this implies that there exists some \emph{deterministic} adversary strategy causing expected regret $\geq \frac{1}{40}(T-1)^{2/3}$, \emph{and} loss drift at most $\sqrt{3\log(t)+18}$ for all $t$.

%
%

\section{Proofs of Upper Bounds}
\label{app:upper}

\begin{proof}[Proof of \thmref{thm:fullUpBound}]
Each loss functions equals $f_t(x_{1:t}) = \ell(x_t) + \Ind{x_t \neq x_{t-1}}$, where
$\ell_t$ is an oblivious loss function. Since the range of $f_t$ is contained in an interval of size $C$, the range of $\ell_t$
must be contained in an interval of size $C-1$. In other words,
$$
\forall x \in \Acal ~~\ell_t(x) - \min_y \ell_t(y) \leq C-1 ~~.
$$
Therefore, by definition, the range of $\ell'_t$ is contained in
the interval $[0,1]$, and the analysis of the FLL algorithm holds.
Namely, if $X_1,X_2,\ldots$ is the sequence of actions chosen by FLL, then, for any $T$
\begin{equation}
\E\left[ \sum_{t=1}^T \ell_t'(X_t)\right] - \min_{x \in \Acal} \sum_{t=1}^T \ell'_t(x) ~=~ \Ocal(\sqrt{T})~~,
\label{eqn:pf1}
\end{equation}
and
\begin{equation}
\E\left[\sum_{t=1}^T \Ind{X_t \neq X_{t-1}}\right] ~=~ \Ocal(\sqrt{T})~~.
\label{eqn:pf2}
\end{equation}
Plugging the definition of $\ell'_t$ into \eqref{eqn:pf1} and rearranging terms, we get
$$
\E\left[\sum_{t=1}^T \ell_t(X_t)\right] - \min_{x \in \Acal} \ell_t(x) ~=~ (C-1) \Ocal(\sqrt{T}) ~~.
$$
Summing the above with
\eqref{eqn:pf2} gives
$$
\E\left[\sum_{t=1}^T f_t(X_{1:t})\right] - \min_{x \in \Acal} f_t(x\ldots x) ~=~ \Ocal(C\sqrt{T})~~.
$$
\end{proof}

\begin{proof}[Proof of \thmref{thm:banditObliviousUpBound}]
Recall that $f_t(x_{1:t}) = \frac{1}{2(C+D)} \left(\ell_t(x_{t}) - \ell_{t-1}(x_{t-1})\right) + \frac{1}{2}$, and note that our assumptions imply that
\begin{align*}
|\ell_t(x_t) - \ell_{t-1}(x_{t-1})| &= |\ell_t(x_t) - \ell_{t-1}(x_t) + \ell_{t-1}(x_t) - \ell_{t-1}(x_{t-1})| \\
&\leq  |\ell_t(x_t) - \ell_{t-1}(x_t)| + |\ell_{t-1}(x_t) - \ell_{t-1}(x_{t-1})| \\
&\leq  D + C ~~.
\end{align*}
Therefore, $f'_t(x_{1:t})$ is always bounded in $[0,1]$ and the
analysis of Exp3.P holds.  Although $f'_t$ is not an oblivious loss,
the \emph{standard} regret bounds for Exp3.P holds against adaptive
adversaries. Namely, if $X_{1:T}$ is the sequence of actions
chosen by Exp3.P, then
$$
\E\left[\sum_{t=1}^{T}f'_t(X_{1:t})-\min_{x\in\Acal}\sum_{t=1}^{T}f'_t(X_{1:t-1},x)\right] ~=~ \Ocal(\sqrt{T})~~.
$$
Using the definition if $f'_t$, the left hand side above can be rewritten as
\begin{align*}
&\frac{1}{2(C+D)}\E\left[\sum_{t=1}^{T}\big(\ell_t(X_{t})-\ell_{t-1}(X_{t-1})\big)
-\min_{x\in\Acal}\sum_{t=1}^{T}\big(\ell_t(x)-\ell_{t-1}(X_{t-1})\big)\right]\\
&~=~ \frac{1}{2(C+D)}\E\left[\sum_{t=1}^{T}\ell_t(X_{t})-\min_{x\in\Acal}\sum_{t=1}^{T}\ell_t(x)\right]~~.
\end{align*}
Therefore,
$$
\E[R_T] ~=~ \E\left[\sum_{t=1}^{T}\ell_t(X_{t})-\min_{x\in\Acal}\sum_{t=1}^{T}\ell_t(x)\right]~=~ 2(C+D) \Ocal(\sqrt{T})~~.
$$
Using the assumption that $D_t = \Ocal\bigl(\sqrt{\log(T)}\bigr)$, we conclude that
$
\E[R_T] ~=~ \widetilde\Ocal(C\sqrt{T})
$.
\end{proof}
\begin{proof}[Proof of \thmref{thm:banditBoundedUpper}]
First, note that, due to the bounded range and drift assumptions, $\fhat_t \in [0,1]$. Also note that
\[
    f_t(x_{t-m:t}) - f_t(x \dots x) = 2\bigl(C+(m+1)D\bigr)\bigl(\fhat_t(x_{t-m:t}) - \fhat_t(x \dots x) \bigr)~.
\]
As previously mentioned, we divide the $T$ rounds into $J$ consecutive
epochs of the same length $T/J$, where $T/J \ge 2K(m+1)$, plus an additional final epoch of
length at most $T/J$. We let $t_j$ denote the index of the first round
in the $j$-th epoch. We run a mini-batched version of the Hedge
algorithm \cite{FS97} over the epochs: at the beginning of each epoch
$j$, Hedge draws an action $X_j \in \Acal$ which is played
consistently throughout the epoch. Now assume that at the end of each
epoch $j$, loss estimates $g_j(x) \in [0,1]$ for each action $x$ are
available such that
\[
    \E\bigl[g_j(x)\bigr] = \frac{1}{T/J-2m-1}\sum_{t=t_j+2m+1}^{t_{j+1}-1} \fhat_t(x \dots x)
\]
where the randomness used to compute each $g_j$ is independent of that used by Hedge to draw $X_j$. At the end of epoch $j$, we feed loss estimates $g_j(x)$ for each $x\in\Acal$ to Hedge. The resulting regret can be bounded as follows,
\begin{align*}
    \sum_{t=1}^{T} & \E\Bigl[ f_t(X_{t-m:t}) - f_t(x \dots x) \Bigr]
\\&\leq
    \sum_{j=1}^J \sum_{t=t_j}^{t_{j+1}-1} \E\Bigl[ f_t(X_{t-m:t}) - f_t(x \dots x) \Bigr] + \frac{CT}{J}
\\ & =
    2\bigl(C+(m+1)D\bigr)\sum_{j=1}^J \sum_{t=t_j}^{t_j+2m} \E\Bigl[ \fhat_t(X_{t-m:t}) - \fhat_t(x \dots x) \Bigr]
\\ & \quad
    + 2\bigl(C+(m+1)D\bigr)\sum_{j=1}^J \sum_{t=t_j+2m+1}^{t_{j+1}-1} \E\Bigl[ \fhat_t(X_j \dots X_j) - \fhat_t(x \dots x) \Bigr] + \frac{CT}{J}
\\ &\le
    2\bigl(C+(m+1)D\bigr)(2m+1)J
\\ & \quad
    + 2\bigl(C+(m+1)D\bigr)\frac{T}{J}\E\left[\sum_{j=1}^J \E\Bigl[g_j(X_j) - g_j(x) \Big| X_j \Bigr] \right] + \frac{CT}{J}
\\ &=
    2\bigl(C+(m+1)D\bigr)(2m+1)J
\\ & \quad
    + 2\bigl(C+(m+1)D\bigr)\frac{T}{J}\E\left[\sum_{j=1}^J \Bigl(g_j(X_j) - g_j(x)\Bigr) \right] + \frac{CT}{J}
\\ &\le
    2\bigl(C+(m+1)D\bigr)(2m+1)J + 4\bigl(C+(m+1)D\bigr)\frac{T}{J}\sqrt{J\ln K} + \frac{CT}{J}~.
\end{align*}
In the last step we applied the known upper bound on the regret of
Hedge with respect to losses $g_j \in [0,1]$, where $K$ is the number
of actions. This is valid if, in particular, losses $g_j$ are
oblivious. We now explain how to obtain oblivious estimates $g_j$
with the desired properties. At the beginning of each epoch $j$, we
use the independent randomization to draw $K$ exploration steps $\{t_x
\,:\, x\in\Acal\}$ from the set $T_j = \{t_j,\dots,t_{j+1}-2m-2\}$
with the property that these steps are well separated. Namely, between
any two $t_x$ and $t_{x'}$ there are at least $2m+1$ consecutive free
time steps in $T_j$. During epoch $j$, when we arrive at step $t_x$ we
freeze Hedge and play action $x_0$ for $m+1$ time steps, then we play
action $x$ for $m+1$ more time steps. We use the two observed losses
$f_{t_x+m}(x_0 \dots x_0)$ and $f_{t_x+2m+1}(x \dots x)$ to compute
$\fhat_{t_x+2m+1}(x \dots x)$. Because the $t_x$ are well separated,
the exploration steps do not interfere with each other. Suppose now
that we can draw these points such that the marginal of each $t_x$ is
uniform in $T_j$. Then
\begin{align*}
    \E\bigl[ \fhat_{t_x+2m+1}(x \dots x) \bigr]
&=
    \frac{1}{T/J-2m-1}\sum_{t=t_j}^{t_{j+1}-2m-2} \fhat_{t+2m+1}(x \dots x)
\\ &=
    \frac{1}{T/J-2m-1}\sum_{t=t_j+2m+1}^{t_{j+1}-1} \fhat_t(x \dots x)~.
\end{align*}
This shows that $\fhat_{t_x+2m+1}(x \dots x)$ is a valid estimate $g_j(x)$. Moreover, for each $x \in \Acal$ the quantity $\fhat_{t_x+2m+1}(x \dots x)$ does not depend on Hedge's action $X_j$ for the current epoch $j$. It does not even depend on Hedge's past actions. Hence, Hedge is indeed run on a set of oblivious losses and the standard regret bound applies.

The last thing to prove is that we can draw $\{t_x \,:\, x\in\Acal\}
\subset T_j$ such that the marginal of each $t_x$ is uniform in
$T_j$. Note that giving equal probability to all well separated
configurations of $\{t_x \,:\, x\in\Acal\}$ does not work, because the
times steps closer to the beginning and to the end of $T_j$ appear in
more configurations (for example, check the case $|T_j|=8$ and
$m=1$). This problem can be fixed simply by arranging the points of
$T_j$ on a circle, so that the first point $t_j$ follows the last
point $t_{j+1}-2m-2$, and then enforcing well-separatedness on the
circle. This makes the sample space completely symmetric, excluding
those configurations of exploration points that exploited border
effects.

The additional regret due to the computation of the $K$ exploration points is $2(m+1)CK$ per epoch. The final regret, including these additional costs, is then bounded by
\[
    2\bigl(C+(m+1)D\bigr)(2m+1)J + 4\bigl(C+(m+1)D\bigr)\frac{T}{J}\sqrt{J\ln K} + \frac{CT}{J} + 2(m+1)CKJ~.
\]
Choosing $J$ of order $T^{2/3}$ concludes the proof.
\end{proof}

\end{document}